\documentclass[sigconf]{acmart}
\AtBeginDocument{%
 }

\setcopyright{acmcopyright}
\copyrightyear{2024}
\acmYear{2024}
\acmConference[RecSys '24]{}{Oct 14--18,
  2024}{Bari, Italy}

\usepackage[utf8]{inputenc} % allow utf-8 input
\usepackage[T1]{fontenc}    % use 8-bit T1 fonts
\usepackage{amsfonts}       % blackboard math symbols

\usepackage{multirow}
\usepackage{bm}
\usepackage{algorithm}
\usepackage{algorithmic}
\usepackage{flafter}
\newcommand{\Ex}{\mathop{\bf E\/}}
\newcommand{\corr}{\mathop{\bf Corr\/}}

\usepackage{subcaption}
\usepackage{wrapfig}
\usepackage[capitalise]{cleveref}

\newtheorem{problem}{Problem}
\newcommand{\R}{\mathbb R}
\newtheorem{theorem}{Theorem}[section]
\renewcommand{\Pr}{{\bf Pr}}

\begin{document}
\title{Learned Ranking Function: From Short-term Behavior Predictions to Long-term User Satisfaction}

\author{Yi Wu$^*$}
\affiliation{%
  \institution{Google Inc}
  \city{Mountain View}
  \country{USA}}
\email{wuyish@google.com}
 
\author{Daryl Chang$^*$}
\affiliation{%
  \institution{Google Inc}
  \city{Mountain View}
  \country{USA}}
\email{dlchang@google.com}

\author{Jennifer She}
\affiliation{%
  \institution{Google DeepMind}
  \city{Mountain View}
  \country{USA}}
\email{jenshe@google.com}

\author{Zhe Zhao}
\affiliation{%
  \institution{University of California, Davis}
  \city{Davis}
  \country{USA}}
\email{zao@ucdavis.edu}

\author{Li Wei}
\affiliation{%
  \institution{Google Inc}
  \city{Mountain View}
  \state{California}
  \country{USA}}
\email{liwei@google.com}

\author{Lukasz Heldt}
\affiliation{%
  \institution{Google Inc}
  \city{Mountain View}
  \state{California}
  \country{USA}}
\email{heldt@google.com}

\renewcommand{\shortauthors}{Yi Wu et al.}

\begin{abstract}
We present the Learned Ranking Function (LRF), a system that
takes short-term user-item behavior predictions as input and outputs a slate of recommendations that directly optimizes for long-term user satisfaction. Most previous work is based on optimizing the hyperparameters of a heuristic function. We propose to model the problem directly as a slate optimization problem with the objective of maximizing long-term user satisfaction. We also develop a novel constraint optimization algorithm that stabilizes objective tradeoffs for multi-objective optimization. We evaluate our approach with live experiments and describe its deployment on YouTube.

\end{abstract}



\maketitle
\def\thefootnote{*}\footnotetext{Equal contribution to the work.}\def\thefootnote{\arabic{footnote}}

\section{Introduction and Related Work} \label{sec:introduction}
\noindent Large video recommendation systems typically have the following stages:
\begin{enumerate}
    \item Candidate Generation: The system first generates a short list of video candidates from a large corpus \cite{seq,gravity}.
    \item Multitask Model Scoring: A Multitask model makes predictions about user behaviors (such as CTR, watch time after click) for all the candidates \cite{wnranking, deepyt}.  
    \item Ranking: Multitask predictions are combined into a single ranking score to sort all candidates \cite{rlmtf,rl_retention,pinterest,instagram,meta,meta_news_feed}.
    \item Re-ranking: Additional logic is applied to ranking score to ensure other objectives, e.g. diversity~\cite{dpp}, taking cross-item interaction into consideration.
\end{enumerate}
This paper primarily focuses on the ranking stage, i.e. combining user behavior predictions to optimize long-term user satisfaction.

Most existing deployed solutions (e.g. Meta~\cite{meta_news_feed,instagram,meta}, Pinterest~\cite{pinterest} and Kuaishou~\cite{rl_retention}) use a heuristic ranking function to combine multitask model predictions. As an example, given input user behavior predictions $s_1,s_2,\ldots, s_k$, the ranking formula can be $\sum_{i=1}^k w_i s_i$ with $w_i$ being the hyperparameters. Then these systems apply hyperparameter search methods (e.g. Bayesian optimization, policy gradient) to optimize the ranking function. Typically the complexity of optimization grows  with the number of hyperparameters, making it hard to change objectives, add input signals, or increase the expressiveness of the combination function.

We formulate the problem as a slate optimization instead. The goal of the optimization is learn a general ranking function to produce a slate that maximizes long-term user satisfaction. 

Let us take a look at related work in the area of slate optimization for long-term rewards. In~\cite{slateq}, the authors propose the SlateQ method, which applies reinforcement learning to solve slate optimization. One limitation of the work is it assumes a simple user interaction model without considering the impact of slate position on the click probability. In~\cite{markov}, the authors give an efficient algorithm for the combinatorial optimization problem of reward maximization under the cascade click model~\cite{cascade,dbn}, assuming the dynamics of the system are given as input. 

Existing slate optimization work typically assumes that future rewards are zero when a user abandons a slate, which is unrealistic. Platforms like video streaming services have multiple recommendation systems (e.g., watch page, home page, search page), where users might abandon one and return later through engagement with another. Hence, it is important to model and optimize the \textbf{lift} value of a slate, i.e. its incremental value over the baseline value of the user abandoning the slate. 

Another less studied but important issue when applying slate optimization at the ranking stage is the \textbf{stability }of multi-objective optimization. Most recommendation systems need to balance trade-offs among multiple objectives. Stability here refers to maintaining consistent trade-offs among these objectives when orthogonal changes, such as adding a feature or modifying the model architecture, are made to the algorithm. The stability is crucial for system reliability and developer velocity.

To address these existing limitations, we present the \textbf{Learned Ranking Function} (LRF) system. Our main contributions are three-fold:
\begin{enumerate}
    \item We model the user-slate interaction as a cascade click model~\cite{cascade} and propose an algorithm to optimize slate-wise long-term rewards. We explicitly model the value of abandonment and optimize for the long-term rewards for entire platform.
    \item We propose a novel constrained optimization algorithm based on dynamic linear scalarization to ensure the stability of trade-offs for multi-objective optimization.
    \item We show how the LRF is fully launched on YouTube and provide empirical evaluation results.
\end{enumerate}

The rest of paper is organized as follows. In Section~\ref{sec:problem}, we define the Markov Decision Process(MDP) for the problem of long-term rewards slate optimization. In Section~\ref{sec:opt}, we propose an optimization algorithm to solve the MDP problem. We show how we deploy the LRF to YouTube with evaluation results in Section~\ref{sec:dep_eval}.

\section{Problem Formation} \label{sec:problem}
\subsection{MDP Formulation}
We model the problem of ranking videos using the following MDP:
\begin{itemize}
    \item state space $\mathcal{S}=\mathcal{U}\times \{V |V\subset\mathcal{V}, |V|=n\}$. Here $\mathcal{U}$ is some user state space and $V$ is a set of $n$ candidate videos nominated for ranking from $\mathcal{V}$, the universe of all videos.
    \item action space $\mathcal{A}$ is all permutations of $n$. The system will rank $V=\{V^1,V^2,\ldots, V^n\}$ by the order of $V^{\sigma(1)}, \ldots, V^{\sigma(n)}$ with action $\sigma\in \mathcal{A}$.
    \item $\mathcal{P}:\mathcal{S} \times \mathcal{A}\times \mathcal{S}\to [0,1]$ is the state transition probability.
    \item reward function $r(s,\sigma)\in \R^m$ is the immediate reward vector by taking action $\sigma$ on state $s$. We consider the general case that there are $m$ different type of rewards.
    \item discounting factor $\gamma\in (0,1)$ and initial state distribution $\rho_0$.
\end{itemize}
A policy  $\pi$ is a mapping from user state $\mathcal{S}$ to a distribution on $\mathcal{A}$.  
Applying policy $\pi$ on $\rho_0$ gives a distribution on user trajectory $\mathcal{D}({\rho_0,\pi})$ defined as follows.
\begin{definition} \label{def:traj} We define $\mathcal{D}({\rho_0,\pi})$ as the distribution of user trajectories when applying policy $\pi$ on initial state distribution $\rho_0$. Here each user trajectory is a list of tuples $((s_0,\sigma_0,c_0),(s_1, \sigma_1,c_1), \ldots, )$. Here $s_i=(u_i,V_i)$ is the user state; $\sigma_i$ is a permutation action applied on $V$; $c_i$ is the user click position (with a value of $0$ indicating no click).
 We define cumulative reward for  $\tau$ starting from timestamp $t$ as $$Reward(\tau,t) = \sum_{t'\geq t}\gamma^{t'-t} \cdot r(s_t,\sigma_t).$$
and cumulative reward for policy $\pi$ as  $$J(\pi) = E_{\tau\sim \mathcal{D}(\rho_0,\pi)}[Reward(\tau, 0)].$$ 
\end{definition}
The optimization problem is to maximize cumulative reward for a primary objective subject to constraints on secondary objectives:
\begin{problem}\label{problem:opt}
$$\max_{\pi} J^{1}(\pi)$$ 
subject to $J^{k}(\pi)\geq \beta_k \text{ for } k=2,3,\ldots,m$.
Here $J^i(\pi)$ is the $i$-th element of $J(\pi)$.
\end{problem}

\subsection{Lift Formulation with Cascade Click model}
Let us follow the standard notation in reinforcement learning and define $Q^{\pi}(s,\sigma)$ as the expected cumulative reward taking action $\sigma$ at state $s$ and applying  policy $\pi$ afterwards; i.e., $$Q^{\pi}(s,\sigma) = r(s,\sigma) + \gamma \Ex_{\tau\sim \mathcal{D}( \mathcal{P}(s,\sigma,\cdot),\pi)}[Reward(\tau, 0)].$$

Below we will factorize $Q^{\pi}(s,\sigma)$ into user-item-functions; i.e., functions that only depend on user and individual item. 
 
Conditional on click position $c$, we can then rewrite $Q^{\pi}(s, \sigma)$ as
\begin{multline*}
    \Pr(c= 0) \cdot \Ex[Q^{\pi}(s,\sigma)|c = 0] +  \sum_{1\leq i\leq n} \Pr(c= i) \cdot \Ex[Q^{\pi}(s,\sigma)|c = i]
\end{multline*}

As mentioned in Section~\ref{sec:introduction}, the reward associated with the user abandoning the slate ($c=0$) can be nonzero.

Notice that  $\sum_{0\leq i\leq n} \Pr(c= i)=1$, so we can further rewrite $Q^{\pi}(s, \sigma)$ as
\begin{multline*}
\Ex[Q^{\pi}(s,\sigma)|c = 0]  \\ + \sum_{1\leq i\leq n} \Pr(c= i) \cdot \left(\Ex[Q^{\pi}(s,\sigma)|c = i] -\Ex[Q^{\pi}(s,\sigma)|c = 0]\right)
\end{multline*}

First, we simplify the term $\Ex[Q^{\pi}(s,\sigma)|c = i]$ for $0\leq i\leq n$ with user-item functions. In order to do so,  we make the "Reward/transition dependence on selection" assumption from~\cite{slateq} which states that future reward only depends on the item the user clicks.  In other words for $s=(u,V)$, 
\begin{enumerate}
\item when $i>0$,  $\Ex[Q^{\pi}(s,\sigma)|c = i]$ can be written  as $R_{clk}^{\pi} (u,V^{\sigma(i)})$ for $R_{clk}^{\pi}$ being a user-item function
\item $\Ex[Q^{\pi}(s,\sigma)|c = 0]$ can be written as $R_{abd}^{\pi} (u)$ for $R_{abd}^{\pi}$ being a user level function.
\end{enumerate}
We further define $R^{\pi}_{lift}$ as
$
R^{\pi}_{lift}(u,v) =  R_{clk}^{\pi}(u,v)-R^{\pi}_{abd}(u),
$ being the difference (i.e., lift) of future rewards associated with the user clicking item $v$ compared to user abandoning the slate. 

Next, we simplify the term $\Pr(c= i)$ with user-item functions by assuming the user interacts with the slate according to a cascade click model~\cite{cascade}. To model the behavior of the user abandoning a slate, we consider a variant~\cite{dbn,markov} which also allows the user to abandon the slate, in addition to skip and click,  when inspecting an item, as illustrated in Figure~\ref{fig:markov_reward_process}:
\begin{figure}[H]
    \centering
    \includegraphics[width=0.9\linewidth]{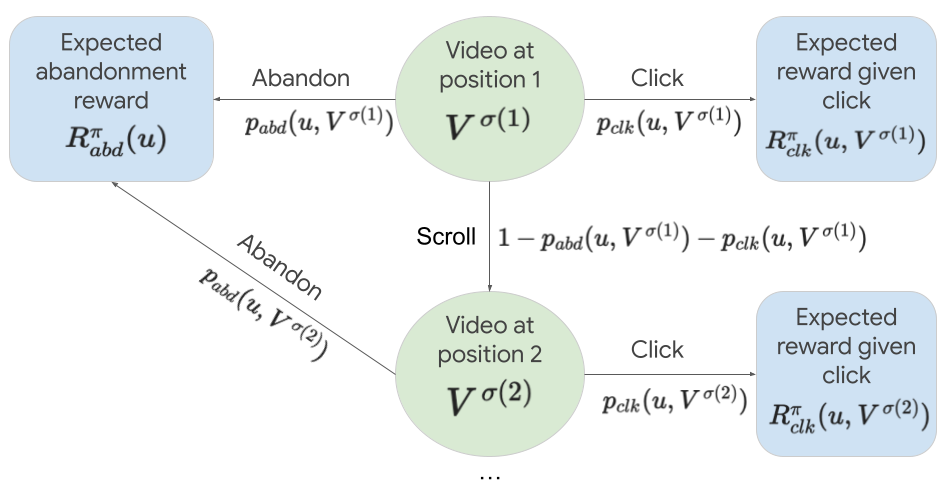}
    \caption{Markov Reward Process with Cascade Click Model}
    \label{fig:markov_reward_process}
\end{figure}
\begin{definition} (Cascade Click Model) \label{def:cascade}
Given user state $s = (u, V)$, where $u$ represents the user and $V$ represents the set of items, and a ranking order $\sigma$ on $V$, the Cascade model describes how a user interacts with a list of items sequentially. The user's interaction with the list when inspecting an item is characterized by the following user-item functions:

\begin{itemize}
\item $p_{clk}$: A user-item function where $p_{clk}(u, v)$ represents the probability of user $u$ clicking on item $v$ when inspecting it.
\item $p_{abd}$: A user-item function where $p_{abd}(u, v)$ represents the probability of user $u$ abandoning the slate when inspecting item $v$.
\end{itemize}

Taking $(u, V), \sigma, p_{clk}, p_{abd}$ as input, the Cascade model defines function $P_{cascade}^i$ that outputs the probability the user clicks on the item at the $i$-th position (for $1 \leq i \leq n$) with the form:

\begin{multline}\label{eqn
}
P_{cascade}^i(p_{clk}, p_{abd}, (u, V), \sigma) = \\
\left(\prod_{j=1}^{i-1} (1 - p_{clk}(u, V^{\sigma(j)}) - p_{abd}(u, V^{\sigma(j)}))\right) \cdot p_{clk}(u, V^{\sigma(i)}).
\end{multline}

The probability that the user abandons the slate without clicking on any items is defined by the function $P_{cascade}^0$ as:

\begin{equation}\label{eqn
}
P_{cascade}^0(p_{clk}, p_{abd}, (u, V), \sigma) = 1 - \sum_{i=1}^{n} P_{cascade}^i(p_{clk}, p_{abd}, (u, V), \sigma).
\end{equation}
\end{definition}

Putting everything together, we can rewrite $Q^\pi(s,\sigma)$ as 
\begin{equation}\label{eqn:slate} R_{abd}^{\pi} (u)  + \sum_{i=1}^{n} P_{cascade}^i(p_{clk},p_{abd},s,\sigma)\cdot R_{lift}^{\pi} (u, V^{\sigma(i)})
\end{equation}

We call equation~(\ref{eqn:slate}) the \textbf{lift formulation with cascade click model}. A natural question is how to order items to maximize $Q^{\pi}(s,\sigma)$ when there is only a single objective. Interestingly, despite the slate nature of this optimization, we prove that the problem can be solved by  a user-item ranking function. 
\begin{theorem} \label{thm:opt} Given user-item functions $p_{clk},p_{abd},R_{abd}^{\pi},R_{lift}^{\pi}$ as input, the optimal ranking for user $u$ on candidate $V$ maximizing $Q^\pi((u,V),\sigma)$  for a \textbf{scalar reward function} is to order all items $v\in V$ by  $\frac{p_{clk}(u,v)}{p_{clk}(u,v)+p_{abd}(u,v)} \cdot R_{lift}^{\pi}(u,v)$.
\end{theorem}
\begin{proof} In equation~(\ref{eqn:slate}), $R_{abd}^{\pi} (u)$ only depends on users. Therefore, it  suffices to  optimize  $\sum_{i=1}^{n} P_{cascade}^i(p_{clk},p_{abd},V,\sigma)\cdot R_{lift}^{\pi} (u, V^{\sigma(i)})$. The rest of the proof follows  Theorem 1 in ~\cite{markov}.
\end{proof}

\section{Optimization Algorithm}\label{sec:opt}
This section outlines the optimization algorithm for solving Problem~\ref{problem:opt}, initially for the special case of a single objective; i.e., $m=1$ and subsequently extending to multi-objective constraint optimization.

\subsection{Single Objective Optimization}
Our algorithm employs an on-policy Monte Carlo approach~\cite{rlbook} which iteratively applies following  two steps:
\begin{enumerate}
\item Training: Build a function approximation $Q(s,\sigma;\theta)$ for $Q^{\pi}(s,\sigma)$ by separately building function approximations for $R^{\pi}_{abd}$, $R^{\pi}_{clk}$, $p_{clk}$ and $p_{abd}$, using data collected by applying some initial policy $\pi$.
\item Inference: Modify the policy $\pi$ to be  $arg\max_{\sigma} Q(s,\sigma;\theta)$ (with exploration).
\end{enumerate}

We outline the main steps in Algorithm~\ref{alg:opt} and discuss technical details in Section~\ref{sec:training} and ~\ref{sec:inference}.
\begin{algorithm}[t!]
\caption{Single objective optimization}
\begin{algorithmic}
\label{alg:opt}
\STATE initialize FIFO data buffer $D$
\STATE initialize network $R_{abd}(u,v;\theta)$, $R_{lift}(u,v;\theta)$, $p_{clk}(u,v;\theta)$, $p_{abd}(u, v;\theta)$ parameterized by $\theta$ and initial policy $\pi$
\WHILE{1}
    \STATE Apply $\pi$ to collect $K$ user trajectories and add it into $D$.
    \STATE Update $\theta$ by using data $D$ (see Section~\ref{sec:training} for details)
     \STATE Update $\pi$ such that for user $u$ with candidate set $V$
    \begin{enumerate}
        \item order $v\in V$ by $ \frac{p_{clk}(u,v;\theta )}{p_{clk}(u,v;\theta)+p_{abd}(u,v;\theta)} \cdot R_{lift}(u,v;\theta)$ 
        \item with probability $\epsilon$,  promote a random  candidate to top
    \end{enumerate}
\ENDWHILE
\end{algorithmic}
\end{algorithm}
\subsubsection{Training}\label{sec:training} The training data is collection of user trajectories (see Definition~\ref{def:traj}) stored in $D$.  Each user trajectory  can be written as $(((u_0,V_0),  \sigma_0,c_0),((u_1,V_1),\sigma_1,c_1), \ldots,.)$. We apply gradient updates for $\theta$ with the following loss functions, in sequential order. 

% See Figure~\ref{fig:ubm_training} and ~\ref{fig:lrf_training} in for graph illustration.

% % Third figure (Will start again in the left column)
% \begin{figure}[H]
%     \centering
%     \includegraphics[width=1.0\linewidth]{UBM training diagram.png}
%     \caption{Illustration of Click Model training on a dummy slate of two items with a click on the first item.}
%     \label{fig:ubm_training}
% \end{figure}

% % Fourth figure (Will appear in the right column)
% \begin{figure}[H]
%     \centering
%     \includegraphics[width=1.0\linewidth]{LRF training diagram.png}
%     \caption{Illustration of the three stages of LRF training.}
%     \label{fig:lrf_training}
% \end{figure}

\paragraph{Training the abandon reward network} $R_{abd}(u;\theta)$ on abandoned pages with MSE  loss function
$$\Ex_{\tau \sim D, t\sim [|\tau|]} \left[|R_{abd}(u_t;\theta)  - Reward(\tau,t)|^2 |c_t=0 \right]$$

\paragraph{Training the  lift reward network} $R_{lift}(u,v;\theta)$ on clicked videos with MSE loss function 
\begin{multline*}\Ex_{\tau \sim D, t\sim [|\tau|]} [|R_{Lift}(u_t, V_t^{\sigma(c_t)};\theta) + R_{abd}(u_t;\theta) - \\  Reward(\tau,t)|^2|  c_t>0]
\end{multline*}
Here we apply the idea from uplift modeling~\cite{uplift} by directly estimating the difference between $R_{clk}^{\pi}(u,v)$ and $R^{\pi}_{abd}(u)$.

\paragraph{Training the click network} on every page with cross-entropy loss function
$$
 \Ex_{\tau \sim D, t\sim [|\tau|]}  [-\log(P_{cascade}^{c_t}(p_{clk}(\cdot,\cdot ;\theta),p_{abd}(\cdot,\cdot;\theta), (u_t, V_t), \sigma_t))],  
 $$using $P_{cascade}^{c_t}$  in Definition~\ref{def:cascade}.

\subsubsection{Inference} \label{sec:inference}
Here we simply apply Theorem~\ref{thm:opt} using the function approximation for $p_{clk},p_{abd},R_{abd}^{\pi},R_{lift}^{\pi}$. We randomly promote a candidate to top with small probability as exploration.

\subsection{Constraint optimization}\label{sec:offline_eval}
When there are multiple objectives, we apply linear scalarization~\cite{scale} to reduce the constraint optimization problem to a unconstrained optimization problem; i.e., we find weights $w_2,w_3,\ldots, w_m$ and define the new reward function as $r^{1} + \sum_{i=2}^m w_i\cdot r^{i}$.  With fixed weight combination, we found it often necessary to search new weights when making changes (e.g., add features, change model architecture) to the system, which slows down our iteration velocity.  We address the problem by dynamically updating $w$ as part of the training. At a high level, we make the following changes to Algorithm~\ref{alg:opt}:
\begin{enumerate}
    \item Training: apply Algorithm~\ref{alg:opt} for $R_{lift}(u,v;\theta)$ as a vector function for all the $m$ objectives separately.
    \item Inference: We find a set of weights $w = (1, w_2,w_3\ldots, w_m)$ and use the following ranking formula at serving time: $$\frac{p_{clk}(u,v;\theta )}{p_{clk}(u,v;\theta)+p_{abd}(u,v;\theta)} \cdot \langle R_{lift}(u,v;\theta),w \rangle.$$ The weights are dynamically updated with offline evaluation.
\end{enumerate}
The algorithm is outlined in Algorithm~\ref{alg:copt}, with details below. 
\subsubsection{Offline evaluation on exploration candidates}\label{sec:offline_eval}  
We apply our offline evaluation on a data set consisting of candidates that is randomly promoted as exploration during serving.
\begin{definition} We define $D_{eval}$ as
$$\{(u,v,r_v)| v \in V \text{ and is randomly promoted}, ((u,V), \sigma, c)\in \tau,\tau  \in D,\}$$
Here $r_v$ is the reward vector $r((u,V),\sigma)$ if $v$ is clicked and $0^m$ otherwise.
\end{definition}
We use $$\corr_{(u,v,r_v)\in D_{eval}}\left(r_v^{i}, \langle R_{lift}(u,v;\theta),w \rangle\right)$$ as the offline evaluation result for $i$-th objectives. Intuitively, we are computing the correlation between weight-combined lift on $v$ with the  immediate rewards from showing $v$. The correlation is computed on exploration candidates to make the evaluation result less biased by the serving policy.

\subsubsection{Optimization with correlation constraint}\label{sec:solve_corr}
With the offline evaluation defined above, we solve the following problem to update $w$ in Algorithm~\ref{alg:copt}.
\begin{problem}\label{problem:corr_opt}
\[
    \min_{w_2,w_3,\ldots,w_m} \Ex_{(u,v,r)\sim D_{eval}}[\sum_{i=2}^m (R_{lift}^{i}(u,v;\theta)\cdot w_i)^2]
\]
such that $\corr_{(u,v,r_v)\in D_{eval}}\left(r_v^{i}, \langle R_{lift}(u,v;\theta),w \rangle\right)\geq \alpha_i$  for $i=2,\ldots, m$ and $w=(1,w_2,\ldots, w_m)$
\end{problem}
Intuitively, we would like to minimize the change to the primary objective while satisfying offline evaluation on secondary objectives.

In the case of a single constraint (i.e., $m=2$),  there is a closed-form solution for the problem. To see this, the optimal solution for $w_2$ must be either $0$ or be a solution that makes the constraint tight; i.e., \begin{equation}\label{eqn:qudratic}
\corr \left(r_v^{2},  R_{lift}^{1}(u,v;\theta) + w_2 \cdot  R_{lift}^{2}(u,v;\theta)\right) =  \alpha_2.
\end{equation}
It is not hard to verify that the solution for  equation~(\ref{eqn:qudratic}) is also the solution for a quadratic equation of $w_2$ that can be solved with closed form. Therefore, we can set $w_2$ be the best feasible solution from $\{0\} \cup \{\text{solution for equation~(\ref{eqn:qudratic}})\}.$

When there is more than one constraint, we found that applying constraints sequentially works  well in practice. One can also apply grid search as the offline evaluation can be done efficiently.

\begin{algorithm}[t!]
\caption{Constraint Optimization}
\begin{algorithmic}
\label{alg:copt}
\STATE initialize FIFO data buffer $D$
\STATE initialize network $R_{abd}(u,v;\theta)$, $R_{lift}(u,v;\theta)$, $p_{clk}(u,v;\theta)$, $p_{abd}(u, v;\theta)$ parameterized by $\theta$ and initial policy $\pi$
\STATE initialize $m$-dimensional weight vectors $w=(1, 0,\ldots, 0)$.
\WHILE{1}
    \STATE apply $\pi$ and add $K$ user trajectory into $D$.
    \STATE update $\theta$ as Algorithm~\ref{alg:opt} 
    \STATE update $w$ (See Section~\ref{sec:offline_eval} and ~\ref{sec:solve_corr})
    \STATE Update $\pi$ such that for user $u$ with candidate set $V$
    \begin{enumerate}
        \item order all $v\in V$
 $\frac{p_{clk}(u,v;\theta)}{p_{clk}(u,v;\theta)+p_{abd}(u,v;\theta)}\cdot  \langle R_{lift}(u,v;\theta),w \rangle$
        \item with  probability $\epsilon$,  promote a random  candidate to top.
    \end{enumerate}
\ENDWHILE
\end{algorithmic}
\end{algorithm}

\section{Deployment and Evaluation}\label{sec:dep_eval}
\begin{figure}[H]
    \centering
    \includegraphics[width=1.0\linewidth]{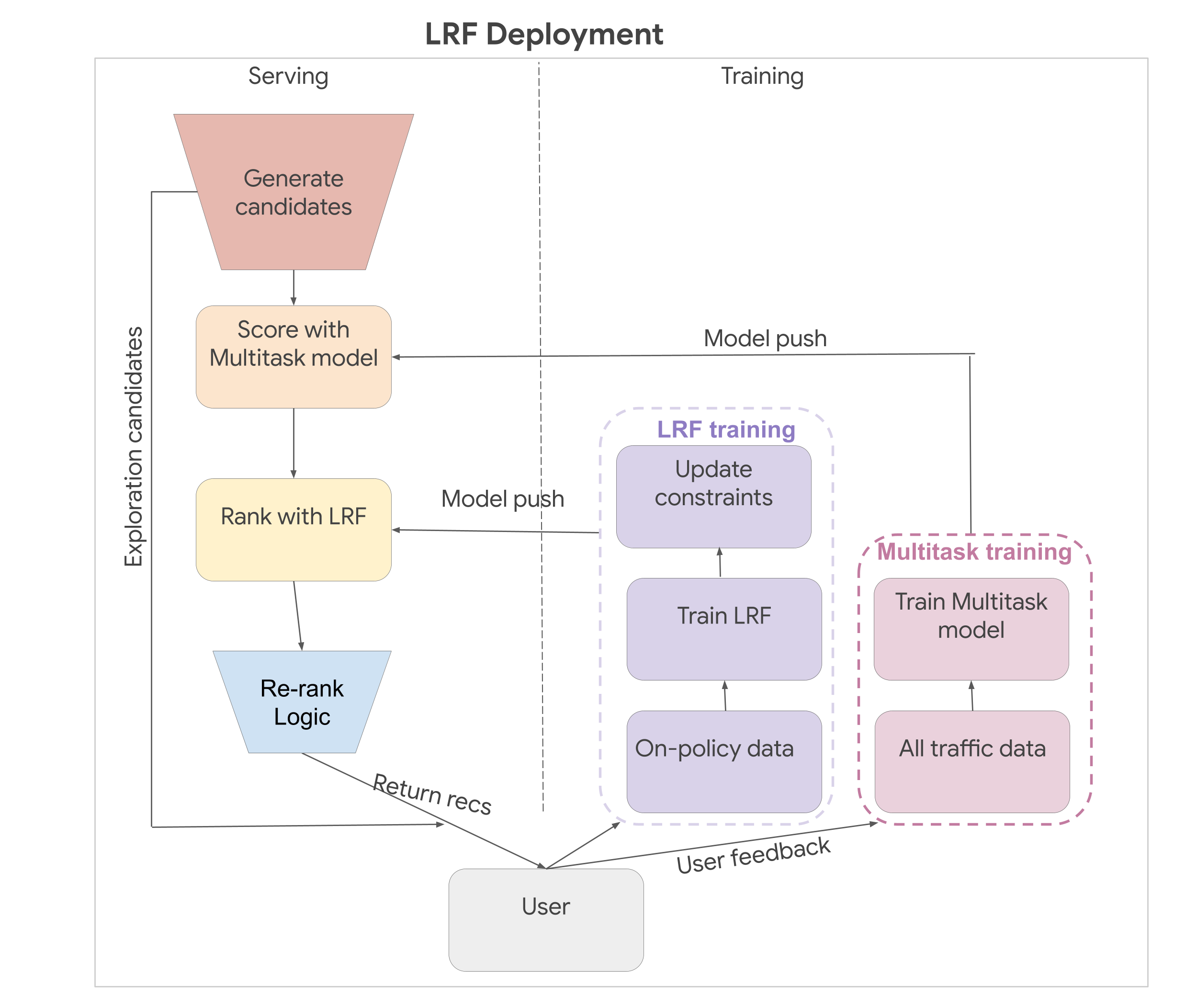}
    \caption{LRF deployment diagram}
    \label{fig:deployment}
\end{figure}

\subsection{Deployment of LRF System}\label{sec:dep}
The LRF was initially launched on YouTube's Watch Page, followed by the Home and Shorts pages. Below we discuss its deployment on Watch Page. Also see illustration in Figure ~\ref{fig:deployment}.

\paragraph{Lightweight model with on-policy training.} The LRF system applies an on-policy RL algorithm. In order to enable evaluating many different LRF models with on-policy training, we make all LRF model training on a small slice (e.g. 1\%) of overall traffic. By doing so,we can compare production and many experimental models that are all trained on-policy together. 

\paragraph{Training and Serving.} The LRF system is continuously trained  with user trajectories from the past few days. Our primary reward function is defined as the user satisfaction on watches, similar to the metric described in Section 6 of \cite{user_satisfaction}. The features include the user behavior predictions from multitask models, user context features (e.g. demographics) and video features (e.g. video topic). We use both continuous features and sparse features with small cardinality. The LRF model comprises small deep neural networks, with roughly $\Theta(10^4)$ parameters. The inference cost of the model is small due to the size of the model. We use the offline evaluation described in Section~\ref{sec:offline_eval} to ensure model quality before pushing to production. At serving time, the LRF takes the aforementioned features as input and outputs a ranking score for all items. 

\subsection{Evaluation}\label{sec:eval}
We conducted A/B experiments for $\Theta(week)$ on YouTube to evaluate the effectiveness of the LRF.  Metric trends are shown in Figure~\ref{fig:all_exps}. Note that first three experiments describe sequential improvements to the production system; the last two experiments ablate certain components of the LRF.

\paragraph{Evaluation Metric} Our primary objective is a metric measuring long-term cumulative user satisfaction; see Sec 6 of \cite{user_satisfaction} for details. 

\paragraph{Baseline before LRF Launch:} The previous system uses a heuristic ranking function optimized  by Bayesian optimization~\cite{google_vizier}. 

\paragraph{Hyperparameters:} We tuned two types of hyperparameters when deploying the LRF: training parameters, such as batch size, are tuned using offline loss; reward parameters, such as constraint weights, are tuned using live experiments.

\begin{figure}
\includegraphics[width=0.3\textwidth]{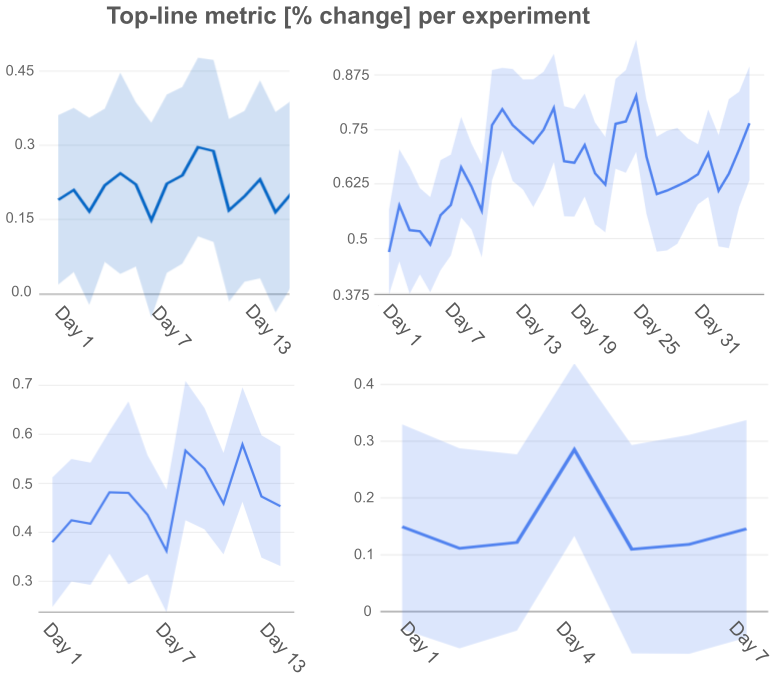}
\caption{Metrics for experiments in Section~\ref{sec:launch}(top left), \ref{sec:cascade} (top right), \ref{sec:uplift} (bottom left), and \ref{sec:two} (bottom right).}
\label{fig:all_exps}
\end{figure}

\subsubsection{Initial Deployment of LRF}
\label{sec:launch}
 \noindent We initially launched a simplified version of the LRF that uses the CTR prediction from the multitask model and ranks all candidates by $CTR\cdot R_{lift}$ . It also uses  a set of fixed weights to combine secondary objectives. The control was the previous production system, a heuristic ranking function tuned using Bayesian optimization. The LRF outperformed the production  by \textbf{0.21\%} with 95\% CI [0.04, 0.38] and was launched to production.

\subsubsection{Launch Cascade Click Model}
\label{sec:cascade}
\noindent After initial deployment of the LRF, we ran an experiment to determine the efficacy of the cascade click model, i.e., replacing $CTR$ with  $\frac{p_{clk}}{p_{clk}+p_{abd}}.$ Adding the cascade click model outperformed the control by \textbf{0.66\%}  with 95\% CI [0.60, 0.72] in the top-line metric and was launched to production. 
\subsubsection{Launch Constraint Optimization}
\noindent We found metric trade-offs between primary and secondary objectives unstable when combining the rewards using fixed weights.  To improve stability, we launched the constraint optimization. As an example of the improvement, for the same model architecture change, we saw a $13.15\%$ change in the secondary objective pre-launch, compared to a 1.46\% change post-launch. This post-launch fluctuation is considered small for that metric.
\subsubsection{Ablating Lift Formulation}
\noindent To determine the necessity of the lift formula, we ran an experiment that set $R_{abd}$ to be $0$.  Such a change regresses top-line metrics by \textbf{0.46\%} with 95\% CI [0.43, 0.49]. The metric contributed from watch page recommendations actually increases by $0.2\%$ with 95\% CI [0.14, 0.26]. This suggests the importance of  lift formulation as it is sub-optimal to only maximize rewards from watch page suggestions.
\subsubsection{Two Model Approach}\label{sec:uplift}
\label{sec:two}
\noindent We make separate predictions for $R_{abd}$ and $R_{clk}$. This is also known as the two-model baseline in uplift modelling~\cite{uplift}. The ranking formula is then $\frac{p_{clk}}{p_{clk}+p_{abd}}\cdot (R_{click}-R_{abd})$. The experiment results show that our production LRF outperforms this baseline in the top-line metric by \textbf{0.12\%} with 95\% CI [0.06, 0.18].

\section{Conclusion}
We presented the Learned Ranking Function (LRF), a system that combines short-term user-item behavior predictions to optimizing slates for long-term user satisfaction. One future direction is to apply more ideas from Reinforcement Learning such as off-policy training and TD Learning~\cite{rlbook}. Another future direction is to incorporate re-ranking algorithm (e.g.,~\cite{dpp,liu2022neural}) into the LRF system.

\newpage
\bibliographystyle{ACM-Reference-Format}
\balance 
\bibliography{sigproc}


\begin{thebibliography}{21}

%%% the Web address.

\ifx \showCODEN    \undefined \def \showCODEN     #1{\unskip}     \fi
\ifx \showDOI      \undefined \def \showDOI       #1{#1}\fi
\ifx \showISBNx    \undefined \def \showISBNx     #1{\unskip}     \fi
\ifx \showISBNxiii \undefined \def \showISBNxiii  #1{\unskip}     \fi
\ifx \showISSN     \undefined \def \showISSN      #1{\unskip}     \fi
\ifx \showLCCN     \undefined \def \showLCCN      #1{\unskip}     \fi
\ifx \shownote     \undefined \def \shownote      #1{#1}          \fi
\ifx \showarticletitle \undefined \def \showarticletitle #1{#1}   \fi
\ifx \showURL      \undefined \def \showURL       {\relax}        \fi
% invisible to TeX
\providecommand\bibfield[2]{#2}
\providecommand\bibinfo[2]{#2}
\providecommand\natexlab[1]{#1}
\providecommand\showeprint[2][]{arXiv:#2}

\bibitem[\protect\citeauthoryear{??}{met}{2019}]%
        {meta_news_feed}
 \bibinfo{year}{2019}\natexlab{}.
\newblock \bibinfo{title}{Combining online and offline tests to improve News
  Feed ranking}.
\newblock
\newblock
\urldef\tempurl%
\url{https://ai.meta.com/blog/online-and-offline-tests-to-improve-news-feed-ranking/}
\showURL{%
\tempurl}


\bibitem[\protect\citeauthoryear{??}{met}{2020}]%
        {meta}
 \bibinfo{year}{2020}\natexlab{}.
\newblock \bibinfo{title}{Efficient tuning of online systems using Bayesian
  optimization}.
\newblock
\newblock
\urldef\tempurl%
\url{https://engineering.fb.com/2018/09/17/ml-applications/bayesian-optimization-for-tuning-online-systems-with-a-b-tests/}
\showURL{%
\tempurl}


\bibitem[\protect\citeauthoryear{Aggarwal, Feldman, Pál, and
  Muthukrishnan}{Aggarwal et~al\mbox{.}}{2008}]%
        {markov}
\bibfield{author}{\bibinfo{person}{Gagan Aggarwal}, \bibinfo{person}{Jon
  Feldman}, \bibinfo{person}{Martin Pál}, {and} \bibinfo{person}{S.
  Muthukrishnan}.} \bibinfo{year}{2008}\natexlab{}.
\newblock \showarticletitle{Sponsored Search Auctions for Markovian Users}. In
  \bibinfo{booktitle}{\emph{Fourth Workshop on Ad Auctions; Workshop on
  Internet and Network Economics (WINE).}}
\newblock
\urldef\tempurl%
\url{http://arxiv.org/abs/0805.0766}
\showURL{%
\tempurl}


\bibitem[\protect\citeauthoryear{Cai, Liu, Wang, Zuo, Xie, Yang, Zheng, Jiang,
  and Gai}{Cai et~al\mbox{.}}{2023}]%
        {rl_retention}
\bibfield{author}{\bibinfo{person}{Qingpeng Cai}, \bibinfo{person}{Shuchang
  Liu}, \bibinfo{person}{Xueliang Wang}, \bibinfo{person}{Tianyou Zuo},
  \bibinfo{person}{Wentao Xie}, \bibinfo{person}{Bin Yang},
  \bibinfo{person}{Dong Zheng}, \bibinfo{person}{Peng Jiang}, {and}
  \bibinfo{person}{Kun Gai}.} \bibinfo{year}{2023}\natexlab{}.
\newblock \showarticletitle{Reinforcing user retention in a billion scale short
  video recommender system}. In \bibinfo{booktitle}{\emph{Companion Proceedings
  of the ACM Web Conference 2023}}. \bibinfo{pages}{421--426}.
\newblock


\bibitem[\protect\citeauthoryear{Chapelle and Zhang}{Chapelle and
  Zhang}{2009}]%
        {dbn}
\bibfield{author}{\bibinfo{person}{Olivier Chapelle} {and} \bibinfo{person}{Ya
  Zhang}.} \bibinfo{year}{2009}\natexlab{}.
\newblock \showarticletitle{A dynamic bayesian network click model for web
  search ranking}. In \bibinfo{booktitle}{\emph{Proceedings of the 18th
  international conference on World wide web}}. \bibinfo{pages}{1--10}.
\newblock


\bibitem[\protect\citeauthoryear{Chen, Beutel, Covington, Jain, Belletti, and
  Chi}{Chen et~al\mbox{.}}{2019}]%
        {seq}
\bibfield{author}{\bibinfo{person}{Minmin Chen}, \bibinfo{person}{Alex Beutel},
  \bibinfo{person}{Paul Covington}, \bibinfo{person}{Sagar Jain},
  \bibinfo{person}{Francois Belletti}, {and} \bibinfo{person}{Ed~H Chi}.}
  \bibinfo{year}{2019}\natexlab{}.
\newblock \showarticletitle{Top-k off-policy correction for a REINFORCE
  recommender system}. In \bibinfo{booktitle}{\emph{Proceedings of the Twelfth
  ACM International Conference on Web Search and Data Mining}}.
  \bibinfo{pages}{456--464}.
\newblock


\bibitem[\protect\citeauthoryear{Christakopoulou, Xu, Zhang, Badam, Potter, Li,
  Wan, Yi, Le, Berg, Dixon, Chi, and Chen}{Christakopoulou
  et~al\mbox{.}}{2021}]%
        {user_satisfaction}
\bibfield{editor}{\bibinfo{person}{Konstantina Christakopoulou},
  \bibinfo{person}{Can Xu}, \bibinfo{person}{Sai Zhang},
  \bibinfo{person}{Sriraj Badam}, \bibinfo{person}{Trevor Potter},
  \bibinfo{person}{Daniel Li}, \bibinfo{person}{Hao Wan},
  \bibinfo{person}{Xinyang Yi}, \bibinfo{person}{Elaine Le},
  \bibinfo{person}{Chris Berg}, \bibinfo{person}{Eric~Bencomo Dixon},
  \bibinfo{person}{Ed~H. Chi}, {and} \bibinfo{person}{Minmin Chen}} (Eds.).
  \bibinfo{year}{2021}\natexlab{}.
\newblock \bibinfo{booktitle}{\emph{Reward Shaping for User Satisfaction in a
  REINFORCE Recommender}}.
\newblock


\bibitem[\protect\citeauthoryear{Covington, , Adams, and Sargin}{Covington
  et~al\mbox{.}}{2016}]%
        {deepyt}
\bibfield{author}{\bibinfo{person}{Paul Covington}, \bibinfo{person}{},
  \bibinfo{person}{Jay Adams}, {and} \bibinfo{person}{Emrin Sargin}.}
  \bibinfo{year}{2016}\natexlab{}.
\newblock \showarticletitle{Deep Neural Networks for YouTube Recommendations}.
  In \bibinfo{booktitle}{\emph{Proceedings of the 10th ACM Conference on
  Recommender Systems}}. \bibinfo{address}{New York, NY, USA}.
\newblock


\bibitem[\protect\citeauthoryear{Craswell, Zoeter, Taylor, and Ramsey}{Craswell
  et~al\mbox{.}}{2008}]%
        {cascade}
\bibfield{author}{\bibinfo{person}{Nick Craswell}, \bibinfo{person}{Onno
  Zoeter}, \bibinfo{person}{Michael~J. Taylor}, {and} \bibinfo{person}{Bill
  Ramsey}.} \bibinfo{year}{2008}\natexlab{}.
\newblock \showarticletitle{An experimental comparison of click position-bias
  models}. In \bibinfo{booktitle}{\emph{Web Search and Data Mining}}.
\newblock
\urldef\tempurl%
\url{https://api.semanticscholar.org/CorpusID:2625350}
\showURL{%
\tempurl}


\bibitem[\protect\citeauthoryear{Engineering}{Engineering}{2023}]%
        {pinterest}
\bibfield{author}{\bibinfo{person}{Pinterest Engineering}.}
  \bibinfo{year}{2023}\natexlab{}.
\newblock \bibinfo{title}{Deep multi-task learning and real-time
  personalization for closeup recommendations}.
\newblock
\newblock
\urldef\tempurl%
\url{https://medium.com/pinterest-engineering/deep-multi-task-learning-and-real-time-personalization-for-closeup-recommendations-1030edfe445f}
\showURL{%
\tempurl}


\bibitem[\protect\citeauthoryear{Golovin, Solnik, Moitra, Kochanski, Karro, and
  Sculley}{Golovin et~al\mbox{.}}{2017}]%
        {google_vizier}
\bibfield{author}{\bibinfo{person}{Daniel Golovin}, \bibinfo{person}{Benjamin
  Solnik}, \bibinfo{person}{Subhodeep Moitra}, \bibinfo{person}{Greg
  Kochanski}, \bibinfo{person}{John Karro}, {and} \bibinfo{person}{D.
  Sculley}.} \bibinfo{year}{2017}\natexlab{}.
\newblock \showarticletitle{Google Vizier: {A} Service for Black-Box
  Optimization}. In \bibinfo{booktitle}{\emph{Proceedings of the 23rd {ACM}
  {SIGKDD} International Conference on Knowledge Discovery and Data Mining,
  Halifax, NS, Canada, August 13 - 17, 2017}}. \bibinfo{publisher}{{ACM}},
  \bibinfo{pages}{1487--1495}.
\newblock
\urldef\tempurl%
\url{https://doi.org/10.1145/3097983.3098043}
\showDOI{\tempurl}


\bibitem[\protect\citeauthoryear{Gutierrez and Gérardy}{Gutierrez and
  Gérardy}{2017}]%
        {uplift}
\bibfield{author}{\bibinfo{person}{Pierre Gutierrez} {and}
  \bibinfo{person}{Jean-Yves Gérardy}.} \bibinfo{year}{2017}\natexlab{}.
\newblock \showarticletitle{Causal Inference and Uplift Modelling: A Review of
  the Literature}. In \bibinfo{booktitle}{\emph{Proceedings of The 3rd
  International Conference on Predictive Applications and APIs}}
  \emph{(\bibinfo{series}{Proceedings of Machine Learning Research})},
  \bibfield{editor}{\bibinfo{person}{Claire Hardgrove}, \bibinfo{person}{Louis
  Dorard}, \bibinfo{person}{Keiran Thompson}, {and} \bibinfo{person}{Florian
  Douetteau}} (Eds.), Vol.~\bibinfo{volume}{67}. \bibinfo{publisher}{PMLR},
  \bibinfo{pages}{1--13}.
\newblock
\urldef\tempurl%
\url{https://proceedings.mlr.press/v67/gutierrez17a.html}
\showURL{%
\tempurl}


\bibitem[\protect\citeauthoryear{Ie, Jain, Wang, Narvekar, Agarwal, Wu, Cheng,
  Chandra, and Boutilier}{Ie et~al\mbox{.}}{2019}]%
        {slateq}
\bibfield{author}{\bibinfo{person}{Eugene Ie}, \bibinfo{person}{Vihan Jain},
  \bibinfo{person}{Jing Wang}, \bibinfo{person}{Sanmit Narvekar},
  \bibinfo{person}{Ritesh Agarwal}, \bibinfo{person}{Rui Wu},
  \bibinfo{person}{Heng-Tze Cheng}, \bibinfo{person}{Tushar Chandra}, {and}
  \bibinfo{person}{Craig Boutilier}.} \bibinfo{year}{2019}\natexlab{}.
\newblock \showarticletitle{SlateQ: A Tractable Decomposition for Reinforcement
  Learning with Recommendation Sets}. In \bibinfo{booktitle}{\emph{Proceedings
  of the Twenty-eighth International Joint Conference on Artificial
  Intelligence (IJCAI-19)}}. \bibinfo{address}{Macau, China},
  \bibinfo{pages}{2592--2599}.
\newblock
\newblock
\shownote{See arXiv:1905.12767 for a related and expanded paper (with
  additional material and authors).}


\bibitem[\protect\citeauthoryear{Liu, Xi, Qin, Sun, Chen, Zhang, Zhang, and
  Tang}{Liu et~al\mbox{.}}{2022}]%
        {liu2022neural}
\bibfield{author}{\bibinfo{person}{Weiwen Liu}, \bibinfo{person}{Yunjia Xi},
  \bibinfo{person}{Jiarui Qin}, \bibinfo{person}{Fei Sun}, \bibinfo{person}{Bo
  Chen}, \bibinfo{person}{Weinan Zhang}, \bibinfo{person}{Rui Zhang}, {and}
  \bibinfo{person}{Ruiming Tang}.} \bibinfo{year}{2022}\natexlab{}.
\newblock \showarticletitle{Neural re-ranking in multi-stage recommender
  systems: A review}.
\newblock \bibinfo{journal}{\emph{arXiv preprint arXiv:2202.06602}}
  (\bibinfo{year}{2022}).
\newblock


\bibitem[\protect\citeauthoryear{Roijers, Vamplew, Whiteson, and
  Dazeley}{Roijers et~al\mbox{.}}{2013}]%
        {scale}
\bibfield{author}{\bibinfo{person}{Diederik~M Roijers}, \bibinfo{person}{Peter
  Vamplew}, \bibinfo{person}{Shimon Whiteson}, {and} \bibinfo{person}{Richard
  Dazeley}.} \bibinfo{year}{2013}\natexlab{}.
\newblock \showarticletitle{A survey of multi-objective sequential
  decision-making}.
\newblock \bibinfo{journal}{\emph{Journal of Artificial Intelligence Research}}
   \bibinfo{volume}{48} (\bibinfo{year}{2013}), \bibinfo{pages}{67--113}.
\newblock


\bibitem[\protect\citeauthoryear{Sutton and Barto}{Sutton and Barto}{2018}]%
        {rlbook}
\bibfield{author}{\bibinfo{person}{Richard~S Sutton} {and}
  \bibinfo{person}{Andrew~G Barto}.} \bibinfo{year}{2018}\natexlab{}.
\newblock \bibinfo{booktitle}{\emph{Reinforcement learning: An introduction}}.
\newblock \bibinfo{publisher}{MIT press}.
\newblock


\bibitem[\protect\citeauthoryear{Vorotilov, Vorotilov, and Shugaepov}{Vorotilov
  et~al\mbox{.}}{2023}]%
        {instagram}
\bibfield{author}{\bibinfo{person}{Vladislav Vorotilov},
  \bibinfo{person}{Vladislav Vorotilov}, {and} \bibinfo{person}{Ilnur
  Shugaepov}.} \bibinfo{year}{2023}\natexlab{}.
\newblock \bibinfo{title}{Scaling the Instagram explore recommendations
  system}.
\newblock
\newblock
\urldef\tempurl%
\url{https://engineering.fb.com/2023/08/09/ml-applications/scaling-instagram-explore-recommendations-system/}
\showURL{%
\tempurl}


\bibitem[\protect\citeauthoryear{Wilhelm, Ramanathan, Bonomo, Jain, Chi, and
  Gillenwater}{Wilhelm et~al\mbox{.}}{2018}]%
        {dpp}
\bibfield{author}{\bibinfo{person}{Mark Wilhelm}, \bibinfo{person}{Ajith
  Ramanathan}, \bibinfo{person}{Alexander Bonomo}, \bibinfo{person}{Sagar
  Jain}, \bibinfo{person}{Ed~H Chi}, {and} \bibinfo{person}{Jennifer
  Gillenwater}.} \bibinfo{year}{2018}\natexlab{}.
\newblock \showarticletitle{Practical diversified recommendations on youtube
  with determinantal point processes}. In \bibinfo{booktitle}{\emph{Proceedings
  of the 27th ACM International Conference on Information and Knowledge
  Management}}. \bibinfo{pages}{2165--2173}.
\newblock


\bibitem[\protect\citeauthoryear{Yi, Yang, Hong, Cheng, Heldt, Kumthekar, Zhao,
  Wei, and Chi}{Yi et~al\mbox{.}}{2019}]%
        {gravity}
\bibfield{author}{\bibinfo{person}{Xinyang Yi}, \bibinfo{person}{Ji Yang},
  \bibinfo{person}{Lichan Hong}, \bibinfo{person}{Derek~Zhiyuan Cheng},
  \bibinfo{person}{Lukasz Heldt}, \bibinfo{person}{Aditee Kumthekar},
  \bibinfo{person}{Zhe Zhao}, \bibinfo{person}{Li Wei}, {and}
  \bibinfo{person}{Ed Chi}.} \bibinfo{year}{2019}\natexlab{}.
\newblock \showarticletitle{Sampling-bias-corrected neural modeling for large
  corpus item recommendations}. In \bibinfo{booktitle}{\emph{Proceedings of the
  13th ACM Conference on Recommender Systems}}. \bibinfo{pages}{269--277}.
\newblock


\bibitem[\protect\citeauthoryear{Zhang, Liu, Dai, Qi, Yuan, Zheng, Huang, and
  Tan}{Zhang et~al\mbox{.}}{2022}]%
        {rlmtf}
\bibfield{author}{\bibinfo{person}{Qihua Zhang}, \bibinfo{person}{Junning Liu},
  \bibinfo{person}{Yuzhuo Dai}, \bibinfo{person}{Yiyan Qi},
  \bibinfo{person}{Yifan Yuan}, \bibinfo{person}{Kunlun Zheng},
  \bibinfo{person}{Fan Huang}, {and} \bibinfo{person}{Xianfeng Tan}.}
  \bibinfo{year}{2022}\natexlab{}.
\newblock \showarticletitle{Multi-Task Fusion via Reinforcement Learning for
  Long-Term User Satisfaction in Recommender Systems}. In
  \bibinfo{booktitle}{\emph{Proceedings of the 28th ACM SIGKDD Conference on
  Knowledge Discovery and Data Mining}} \emph{(\bibinfo{series}{KDD ’22})}.
  \bibinfo{publisher}{ACM}.
\newblock
\urldef\tempurl%
\url{https://doi.org/10.1145/3534678.3539040}
\showDOI{\tempurl}


\bibitem[\protect\citeauthoryear{Zhao, Hong, Wei, Chen, Nath, Andrews,
  Kumthekar, Sathiamoorthy, Yi, and Chi}{Zhao et~al\mbox{.}}{2019}]%
        {wnranking}
\bibfield{author}{\bibinfo{person}{Zhe Zhao}, \bibinfo{person}{Lichan Hong},
  \bibinfo{person}{Li Wei}, \bibinfo{person}{Jilin Chen},
  \bibinfo{person}{Aniruddh Nath}, \bibinfo{person}{Shawn Andrews},
  \bibinfo{person}{Aditee Kumthekar}, \bibinfo{person}{Maheswaran
  Sathiamoorthy}, \bibinfo{person}{Xinyang Yi}, {and} \bibinfo{person}{Ed
  Chi}.} \bibinfo{year}{2019}\natexlab{}.
\newblock \showarticletitle{Recommending what video to watch next: a multitask
  ranking system}. In \bibinfo{booktitle}{\emph{Proceedings of the 13th ACM
  Conference on Recommender Systems}}. \bibinfo{pages}{43--51}.
\newblock


\end{thebibliography}

\end{document}